\PassOptionsToPackage{table}{xcolor}
\documentclass[sigconf,nonacm]{acmart}
\usepackage{algorithm}
\usepackage{algorithmic}
\graphicspath{{Image/}}
\AtBeginDocument{%
  }

\copyrightyear{2026}
\acmYear{2026}
\setcopyright{cc}
\setcctype{by}
\acmConference[KDD 2026]
{Proceedings of the 32nd ACM SIGKDD Conference on Knowledge Discovery and Data Mining V.2}
{August 9--13, 2026}
{Jeju Island, Republic of Korea.}
\acmBooktitle{Proceedings of the 32nd ACM SIGKDD Conference on Knowledge Discovery and Data Mining V.2 (KDD 2026), August 9--13, 2026, Jeju Island, Republic of Korea}
\acmISBN{979-8-4007-2259-2/2026/08}
\acmDOI{10.1145/3770855.3818908}
\begin{document}

\title{APCyc: Property-Informed Design of Cyclic Peptides via Automated Cyclization}
\author{Yifan Zhao}
\authornote{Both authors contributed equally to this research.}
\email{yzhao642@connect.hkust-gz.edu.cn}
\affiliation{%
  \institution{\mbox{AI Thrust}}
  \institution{The Hong Kong University of Science and Technology (Guangzhou)}
  \city{Guangzhou}
  \country{China}
}

\author{Lang Qin}
\email{lqin969@connect.hkust-gz.edu.cn}
\authornotemark[1]
\affiliation{%
  \institution{AI Thrust}
  \institution{The Hong Kong University of Science and Technology (Guangzhou)}
  \city{Guangzhou}
  \country{China}
}
\author{Jintai Chen}
\email{jintaichen@hkust-gz.edu.cn}
\authornote{Corresponding author.}
\affiliation{%
  \institution{\mbox{AI-Peptide Drug Design Joint Laboratory}}
  \institution{The Hong Kong University of Science and Technology (Guangzhou)}
  \city{Guangzhou}
  \country{China}
}
\renewcommand{\shortauthors}{Yifan Zhao, Lang Qin, and Jintai Chen}

\begin{abstract}
Cyclic peptides represent a promising class of therapeutic compounds in modern drug discovery, often offering improved stability and binding affinity.
However, the \textit{de novo} design of cyclic peptides remains challenging because methods must identify pocket-adaptive cyclization patterns and linkage sites while simultaneously controlling drug-relevant properties.
This challenge is particularly pronounced for recent generative models trained predominantly on linear peptide data, which may fail to capture cyclization-specific constraints. 
To address the limitation, we introduce \textbf{APCyc}, a target-aware \textit{de novo} cyclic peptide generation framework that explicitly models cyclization and jointly optimizes multiple essential physicochemical properties. 
By using an expanded \textbf{residue vocabulary} and explicitly encoding \textbf{cyclization-site and linkage-type information}, APCyc learns cyclization-aware representations and leverages \textbf{Bayesian posterior guidance} to steer sampling toward cyclic peptides satisfying multiple property objectives. 
Experimental results demonstrate that our model learns target-dependent cyclization preferences, and enables effective and controllable multi-property optimization for cyclic peptide design. The source code of this paper is available at
https://github.com/HKUSTGZ-ML4Health-Lab/APCyc.
\end{abstract}


\maketitle

\begin{center}
\small
Accepted at the 32nd ACM SIGKDD Conference on Knowledge Discovery and Data Mining (KDD 2026).
DOI: \href{https://doi.org/10.1145/3770855.3818908}{10.1145/3770855.3818908}.
\end{center}

\section{Introduction}

Cyclic peptides are short and synthetically accessible amino acid sequences characterized by macrocyclic structural constraints that pre-organize the peptide backbone into bioactive conformations~\citep{fosgerau2015peptide,ji2024cyclic}, thereby minimizing the entropic cost of binding to target proteins and making them attractive for therapeutic applications
\citep{vlieghe2010synthetic,lau2018therapeutic,driggers2008exploration, zorzi2017cyclic,liu2024cyclicpepedia}.
The advantages have motivated the development of computational methods for cyclic peptide discovery, with recent AI-based approaches increasingly used to model sequence--structure--property relationships and prioritize candidates from large combinatorial sequence spaces. However, designing therapeutically useful cyclic peptides remains challenging because candidate molecules must satisfy multiple coupled drug-relevant constraints beyond target binding, including solubility, permeability, protease resistance, metabolic stability, and safety-related properties \citep{otvos2014current}. Thus, cyclic peptide design is inherently a multi-objective optimization problem that requires methods capable of balancing therapeutic constraints rather than optimizing activity alone.

Recent target-aware peptide design methods have primarily focused on generating linear peptides conditioned on target (receptor) sequence and pocket geometry. Within this paradigm, geometry-aware generative models, including latent diffusion approach~\citep{kong2024full} or multimodal torsional flow matching~\cite{lin2024ppflow, li2024full},
have demonstrated the capability to perform full-atomic linear peptide design at target interfaces. Concurrently, approaches utilizing autoregressive generation \citep{li2024pephar}, reinforcement learning paradigms
\citep{oh2025denovo}, 
or pocket-conditioned diffusion models \citep{sayuti2025tldm}, have furthered the effective design of linear peptide binders across diverse targets~\citep{bhat2025novo}. However, these frameworks generally treat cyclization as a post-hoc structural modification rather than an integrated design variable~\citep{li2024cycpeptmp,rettie2025alphafold2cyclic}, failing to capture the complex, target-specific topological requirements inherent in high-affinity binding~\citep{rettie2025alphafold2cyclic,rettie2025macrocycle, wang2024diffpepbuilder}. 

Beyond this lack of structural integration~\citep{jiang2025cpcomposer}, existing frameworks rely on rigid, a priori cyclization heuristics~\citep{zhou2025harmonicsde}.
In these methods, the cyclization type and linkage site positions are typically predefined based on heuristics, prior assumptions, or trial-and-error, rather than inferred directly from the structural requirements of the target binding pocket. This decoupling is fundamentally flawed, as binding sites exhibit substantial geometric diversity, ranging from deep and groove-like pockets to flat or shallow surfaces
\citep{guo2015identification,garcia2023macrocycles} that inherently dictates the optimal peptide topology~\citep{keeling2016key}.
As a result, identifying the pocket-adaptive cyclization strategy in practical cyclic peptide discovery often requires costly enumeration of candidate cyclization strategies and iterative screening \citep{li2022cyclic,kim2025exploring}, motivating the development of automated models for selecting cyclization strategies conditioned on the target binding site.
Meanwhile, key drug-relevant properties such as permeability are tightly coupled to cyclization strategy, with the optimal choice varying across targets
\citep{chandramohan2024design,ji2024cyclic}.
In present methods, essential drug-like properties of cyclic peptides, including membrane permeability, solubility, protease resistance, and immunogenicity, are largely overlooked, posing additional challenges for therapeutic peptide design.

To tackle the above challenges jointly, we propose \textbf{APCyc},
a generative framework that enables \textbf{A}utomated generation of cyclization linkage types and sites, as well as \textbf{P}roperty-informed optimization for \textbf{Cyc}lic peptide design.
Unlike existing approaches that rely on fixed cyclization strategies
\citep{zhou2025harmonicsde,jiang2025cpcomposer},
APCyc explicitly formulates cyclization topology selection as a discrete, target-conditioned decision-making problem.
This design enables the automated generation of cyclization types and residue-level linkage sites conditioned on the receptor pocket context.
Specifically, APCyc decomposes cyclization topology into residue-level participation signals and pairwise linkage signals, allowing the model to explicitly learn cyclized residues, linkage sites, and linkage types during training.
We first extend the peptide residue vocabulary by explicitly distinguishing residues involved in cyclization from non-cyclized residues, even when they share the same amino-acid identity.
These cyclization-aware embeddings are then incorporated into the denoising network of a latent diffusion model \citep{kong2024full}, enabling supervised cyclization signals to condition and guide the denoising trajectory.
For property-informed optimization, we train energy-based surrogates in latent space,
which guide the diffusion process via gradients toward joint multi-property objectives.  
Figure~\ref{fig:teaser} summarizes the high-level design concept of APCyc.
We summarize our contributions as follows:
\begin{itemize}
    \item \textbf{Automated pocket-conditioned cyclization}
    Unlike prior methods that require predefined cyclization types or linkage sites, APCyc learns to select both directly from the target binding-pocket context.
    \item \textbf{Property-informed controllable generation}
    APCyc uses Bayesian posterior guidance to steer cyclic peptide generation across affinity, permeability, protease resistance, solubility, and immunogenicity, achieving the best permeability proxy ($0.107$) and protease-resistance score ($-1.474$).
    \item \textbf{End-to-end cyclic peptide design}
    APCyc unifies cyclization selection, all-atom generation, and property-informed guidance. Across guidance settings, APCyc achieves leading structural quality, including the best Rosetta total score ($-758.545$) and consistency ($0.971$), while maintaining strong binding affinity.
\end{itemize}

\begin{figure}[t]
    \centering
    \includegraphics[width=0.96\columnwidth]{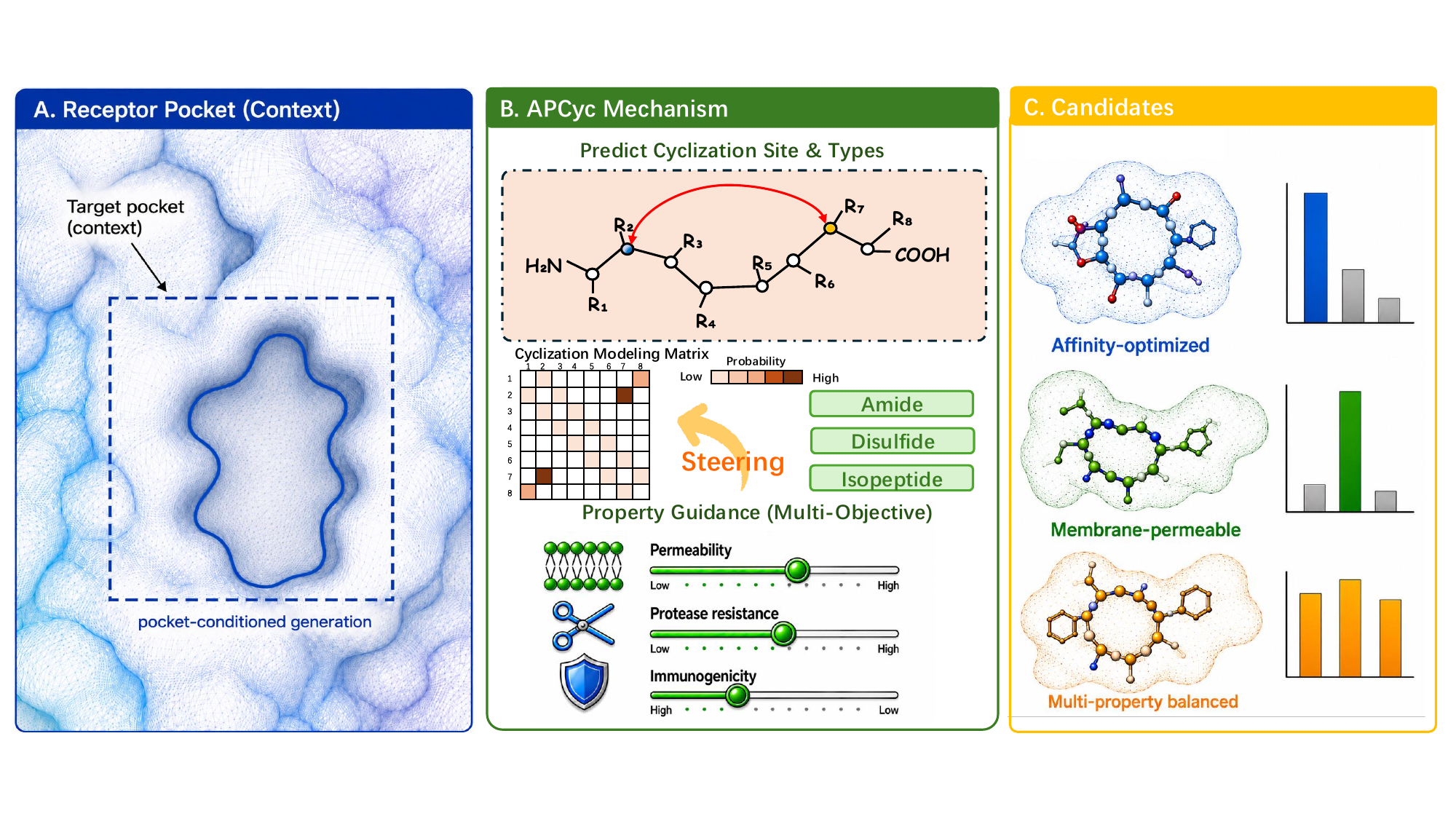}
    \caption{\textbf{APCyc Concept.}
    \textbf{(A)} The receptor pocket provides the structural context for target-conditioned cyclic peptide generation.
    \textbf{(B)} APCyc explicitly models cyclization as a discrete decision process by predicting both cyclization sites and linkage types. A cyclization modeling matrix encodes pairwise site probabilities, while chemically distinct linkage modes, including amide, disulfide, and isopeptide bonds, are incorporated into the generation process. Multi-objective property guidance further steers sampling toward desirable therapeutic profiles, such as membrane permeability, protease resistance, and low immunogenicity.
    \textbf{(C)} APCyc generates diverse candidate cyclic peptides under different design property objectives.}
    \Description{Teaser illustration showing receptor pocket context, APCyc cyclization and property guidance mechanisms, and candidate cyclic peptides optimized for affinity, membrane permeability, and multi-property balance.}
    \label{fig:teaser}
    \vspace{0.7em}
\end{figure}

\section{Related Work}
\subsection{Target-aware cyclic peptide design} 
Traditional cyclic peptide discovery has relied largely on experimental screening methods, such as phage display technologies 
\citep{heinis2009phage}, 
and rational macrocyclization guided by medicinal-chemistry heuristics 
\citep{driggers2008exploration,zorzi2017cyclic,deyle2017phage}. 
However, these approaches are labor-intensive and offer limited ability to jointly explore sequence, structure, and cyclization-topology spaces
\citep{newman2016natural}.
With the rise of deep learning and advances in protein structure prediction models such as AlphaFold2
\citep{jumper2021highly},
computational methods for cyclic peptide design have begun to emerge. 
AfCycDesign \citep{rettie2025alphafold2cyclic} adapts AlphaFold2 to cyclic peptides by encoding N-to-C terminal cyclization through modified positional encodings, enabling structure prediction, sequence redesign, and \textit{de novo} hallucination of cyclic peptides.  
In parallel, diffusion-based generative models built on equivariant architectures \citep{satorras2021n} have improved geometric modeling for cyclic peptide generation.
RFpeptides \citep{rettie2025macrocycle} further extends deep generative design to protein-binding macrocycles through a denoising diffusion-based pipeline.
DiffPepBuilder \citep{wang2024diffpepbuilder}
applies SE(3)-equivariant diffusion, with cyclization handled by post hoc disulfide bonding. Furthermore, CPSDE
\citep{zhou2025harmonicsde} 
utilizes chemical graph-based harmonic SDEs with atom–bond modeling, while CP-Composer
\citep{jiang2025cpcomposer}
performs cyclic peptide design through geometric constraint composition.
While effective, these methods are mainly trained on linear peptides and usually require cyclization constraints to be fixed a priori.
In contrast, APCyc treats cyclization topology as a target-conditioned design variable, allowing it to infer cyclization types and linkage sites from the receptor pocket context.

\subsection{Gradient-guided Bayesian diffusion}
Diffusion models support controllable generation by modifying the reverse denoising process through guidance signals. 
Early work introduced classifier guidance \citep{CG}, 
which steers sampling using gradients from a separately trained classifier that predicts the desired condition or attribute.
\citet{CFG} proposed classifier-free guidance, which enables conditional generation by combining conditional and unconditional score estimates without requiring an explicit classifier.
Beyond categorical conditioning, guidance has been extended to energy-based formulations, where surrogate model gradients guide sampling toward desired objectives under Bayesian posterior interpretation \citep{energy_guide,posteriorsampling}. Further extensions have explored multi-constraint guidance, where multiple objectives or constraints are combined through score composition.
For example, \citet{bansal2023universal} proposed universal guidance, which controls diffusion sampling using arbitrary guidance functions or off-the-shelf auxiliary networks, while MolJO \citep{qiu2024empower} applies gradient-based guidance to joint structure-based molecular optimization.
Our work utilizes regressor-based gradient guidance, combining gradients from multiple physicochemical property surrogates as posterior guidance terms for joint optimization of drug-relevant properties.

\begin{figure*}[t]
    \centering
    \includegraphics[width=0.94\textwidth]{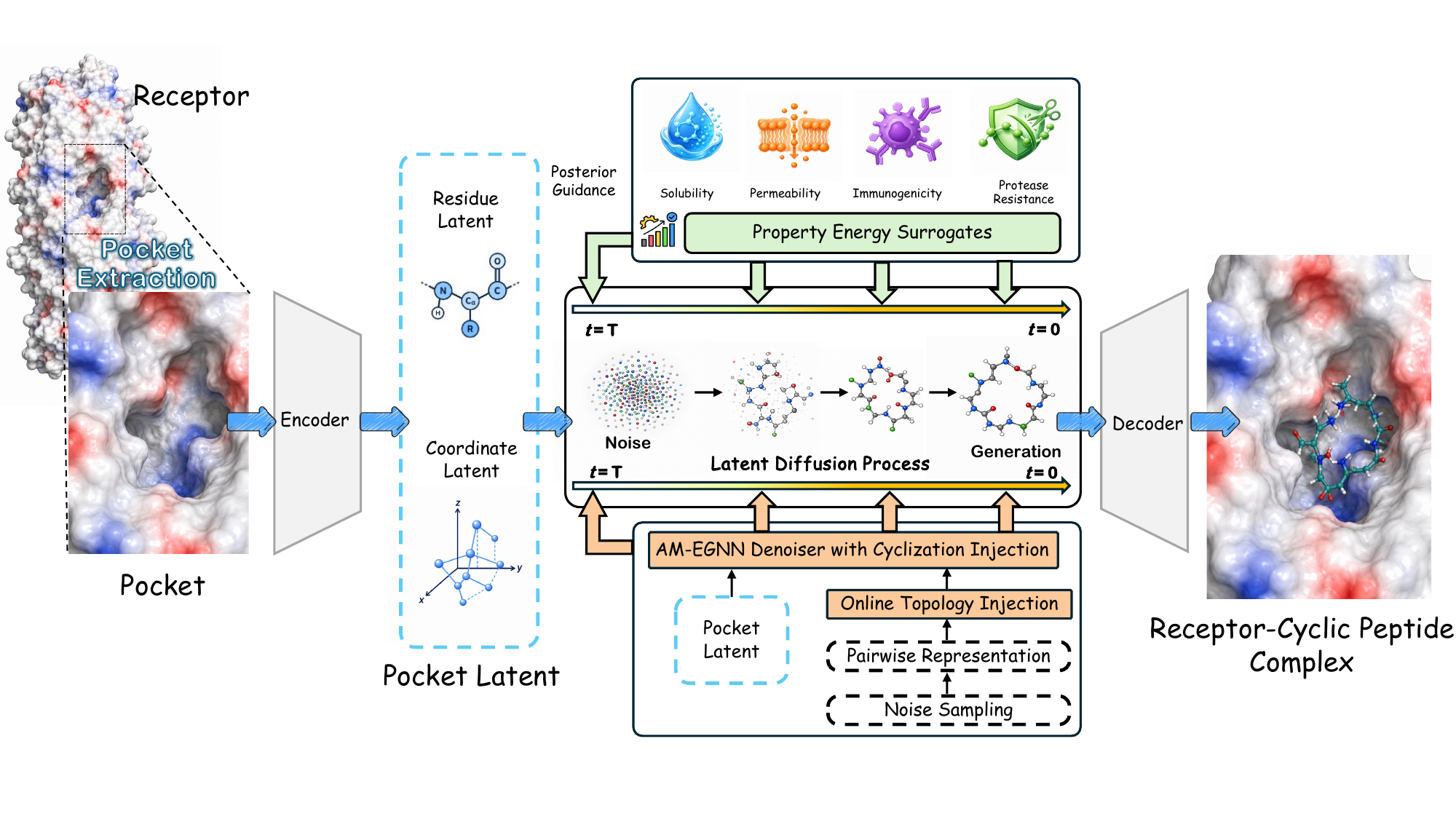}
    \vspace{0.25em}
    \caption{
    \textbf{Sampling workflow of APCyc for pocket-conditioned cyclic peptide design.}
    APCyc extracts a local receptor pocket from the input receptor structure and uses it as a fixed condition.
    Cyclic peptide variables, including residue identities, 3D coordinates, and cyclization topology, are encoded into residue and coordinate latents.
    A latent diffusion process with an AM-EGNN denoiser performs reverse sampling under the fixed pocket condition, with topology injection and energy guidance jointly steering cyclic peptide generation.
    The decoder reconstructs the final cyclic peptide ligand in the fixed receptor pocket.
    }
    \Description{Overview diagram of APCyc encoding receptor-peptide complexes, denoising latent variables with topology injection and posterior guidance, and decoding full-atom cyclic peptides.}
    \label{overview}
    \vspace{0.55em}
\end{figure*}

\section{Method}
In this section, we present \textbf{APCyc},
a joint sequence–structure latent diffusion framework that enables automated cyclization-topology modeling and property-informed generation for cyclic peptides.
Section~\ref{sec:preliminaries} introduces the notation and formalizes target-conditioned cyclic peptide generation, then presents the joint sequence--structure latent diffusion formulation.
Building on this formulation, Section~\ref{sec:cyc_injection} introduces \textbf{Automated Cyclization Pair Injection}, which injects cyclization-aware pair representations into the denoising network, enabling cyclization types and linkage sites to be inferred from the target context during generation.
We then give the diffusion process in detail in Section~\ref{sec:diffusion}.
To promote desired drug-relevant property profiles, Section~\ref{sec:guidance} presents \textbf{Bayesian Posterior Guidance}, which combines gradients from property surrogate models to steer the diffusion process toward joint multi-property objectives.
Finally, Section~\ref{sec:training} describes the training objectives and the complete inference algorithm. The overview of our method is shown in Figure~\ref{overview}.

\subsection{Preliminaries}
\label{sec:preliminaries}

\textbf{Geometric Graph Representation}
We represent the protein--peptide complex as a residue-level geometric graph 
$\mathcal{G} = (\mathcal{N}, \mathcal{E})$, 
where $\mathcal{N}$ denotes the set of residue nodes from both the receptor and the peptide, 
and $\mathcal{E}$ denotes edges encoding sequential, spatial, and covalent relationships between residues.
Each node $i \in \mathcal{N}$ is associated with a tuple 
$\mathbf{n}_i = (s_i, \mathbf{X}_i)$, 
where $s_i$ denotes the residue token, 
$\mathbf{X}_i \in \mathbb{R}^{K \times 3}$ stores padded full-atom coordinates.
Here, $K$ is the maximum number of atoms per residue.


Following previous work
\citep{kong2024full,lin2024ppflow}, 
we decompose the complex into a receptor graph $\mathcal{R}$ and a peptide graph $\mathcal{P}$.
The receptor context $\mathcal{C} \subseteq \mathcal{R}$ is defined as the binding-pocket residues within 10~\AA{} of the reference peptide during training, and is treated as a fixed conditioning input during generation.
We represent the cyclic peptide as 
$\mathcal{P} = (\mathbf{s}, \mathbf{X})$, 
where $\mathbf{s}$ denotes the peptide sequence, $\mathbf{X}$ denotes the full-atom peptide structure. 

In addition to sequence and structure, we denote the desired drug-relevant property targets as 
$\mathbf{Y} = (y_1, \ldots, y_M) \in \mathbb{R}^M$, 
where each $y_m$ is a normalized scalar score for a physicochemical or developability-related property, such as membrane permeability, solubility, or protease resistance.
Given a fixed receptor context $\mathcal{C}$ and desired property targets $\mathbf{Y}$, conditional cyclic peptide generation aims to sample peptide graphs from
$p(\mathcal{P} \mid \mathcal{C}, \mathbf{Y})$, 
where sequence, structure, and cyclization linkage are generated jointly.


\noindent \textbf{Latent Diffusion Modeling}
Following prior work on antibody and linear peptide design
\citep{kong2023end,kong2024full},
we adopt a dyMEAN-style autoencoder to map peptide graphs into residue-level latent representations.
Specifically, an encoder $\mathcal{E}_\phi$ embeds the peptide graph into latent variables
$\mathcal{Z}=\{(z_i,\tilde z_i)\}_{i=1}^L$,
where $z_i$ is E(3)-invariant and $\tilde z_i$ is E(3)-equivariant.
We then learn a diffusion model over $\mathcal{Z}$ using an equivariant GNN denoiser,
and decode the final latents back to peptide sequence and structure via a decoder $\mathcal{D}_\xi$.

\noindent \textbf{Cyclization-Enhanced Vocabulary Expansion}
Motivated by explicit cyclization annotations in cyclic peptide datasets such as CPSea
\citep{yang2025cpsea}, 
we extend the standard residue alphabet $\mathcal{A}_{\mathrm{std}}$ to 
$\mathcal{A}_{\mathrm{ext}} = \mathcal{A}_{\mathrm{std}} \cup \mathcal{A}_{\mathrm{cyc}}$, 
where $\mathcal{A}_{\mathrm{cyc}}$ contains specialized tokens for residues participating in specific cyclization linkages, such as CYS\_SS for disulfide-linked cysteine and LYS\_ISO/ASP\_ISO for residues involved in isopeptide bonds.
This extension allows the model to distinguish residues involved in cyclization from their non-cyclized counterparts at the token level, even when they share the same standard amino-acid identity.

\subsection{Automated Cyclization Pair Injection}
\label{sec:cyc_injection}
At a discrete diffusion timestep $t \in \{0,\ldots,T\}$, the noisy state
$\mathcal{Z}_t = \{(z_{i,t}, \tilde{z}_{i,t})\}_{i=1}^{L}$
is sampled from the closed-form marginal of the DDPM-style latent diffusion
process defined in Sec.~\ref{sec:diffusion}:
\begin{equation}
    z_{i,t} = \alpha_t z_{i,0} + \sigma_t \epsilon_i,
    \quad
    \tilde{z}_{i,t} = \alpha_t \tilde{z}_{i,0} + \sigma_t \tilde{\epsilon}_i,
    \label{eq:ddpm_marginal_small_z}
\end{equation}
where $\epsilon_i \sim \mathcal{N}(\mathbf{0},\mathbf{I})$ is Gaussian noise
for the invariant latent, $\tilde{\epsilon}_i \sim \mathcal{N}(\mathbf{0},\mathbf{I})$
is isotropic Gaussian noise with the same shape as $\tilde{z}_{i,0}$, and
$\alpha_t$ and $\sigma_t$ are the cumulative signal and noise coefficients
satisfying $\sigma_t^2 = 1-\alpha_t^2$.

\noindent \textbf{Pair Representation Construction}
To facilitate cyclization inference, we construct a high-dimensional pairwise tensor $\mathbf{H}^{\text{pair}}_t \in \mathbb{R}^{L \times L \times d_p}$. We leverage both the invariant node latents $z_{i,t}$ and the equivariant coordinates $\tilde{z}_{i,t}$. The pairwise feature vector for residues $(i, j)$ is computed as:
\begin{equation}
    \mathbf{h}_{ij, t}^{\text{pair}} = \phi_{\text{enc}}\left( \text{Concat}\left(z_{i,t}, \, z_{j,t}, \, z_{i,t} \odot z_{j,t}, \, \psi(\|\tilde{z}_{i,t} - \tilde{z}_{j,t}\|) \right) \right)
    \label{eq:pair_rep}
\end{equation}
where $\odot$ denotes the Hadamard product and $\psi(\cdot)$ represents a sinusoidal RBF embedding \citep{rbf}. 
This representation captures both the semantic compatibility (via $z$) and geometric proximity (via $\tilde{z}$) required for bond formation. To disentangle geometric modeling from cyclization reasoning,
we project $\mathbf{h}_{ij, t}^{\text{pair}}$ into two task-specific subspaces:

\begin{equation}
    \mathbf{h}_{ij,t}^{\text{struct}}
    = \phi_{\text{struct}}(\mathbf{h}_{ij,t}^{\text{pair}}),
    \quad
    \mathbf{h}_{ij,t}^{\text{cyc}}
    = \phi_{\text{cyc}}(\mathbf{h}_{ij,t}^{\text{pair}}).
    \label{eq:projection_heads}
\end{equation}
where $\phi_{\text{struct}}$ and $\phi_{\text{cyc}}$ are separate projection heads (MLPs). The resulting $\mathbf{h}_{ij}^{\text{struct}}$ is utilized for geometric updates, while $\mathbf{h}_{ij}^{\text{cyc}}$ serves as the basis for predicting cyclization probability.

We then predict the cyclization type distribution $\mathbf{p}_{\tau,t}$ and the
linkage matrix $\mathbf{S}_t$ at timestep $t$. Chemical constraints, such as
Cys--Cys compatibility, are enforced through a validity mask
$\mathbf{M}_{\mathrm{valid}}\in\{0,1\}^{L\times L}$. Distinct from row-wise
attention, the linkage probabilities are computed by a global normalization over
all valid residue-pair entries. Specifically, the cyclization head predicts
pairwise site logits
$\mathbf{L}_{\mathrm{site},t}\in\mathbb{R}^{L\times L}$ from
$\mathbf{h}_{ij,t}^{\mathrm{cyc}}$. We first compute a masked, pre-symmetrized
probability matrix $\mathbf{P}_t$ as
\begin{equation}
    \mathrm{vec}(\mathbf{P}_t)
    =
    \mathrm{Softmax}
    \left(
    \mathrm{vec}
    \left(
    \mathbf{L}_{\mathrm{site},t}
    +
    \mathcal{M}(\mathbf{M}_{\mathrm{valid}})
    \right)
    \right),
    \label{eq:topo_softmax}
\end{equation}
where $\mathrm{vec}(\cdot)$ flattens an $L\times L$ matrix into a vector,
and the mask operator $\mathcal{M}$ is defined element-wise as
\begin{equation}
    \mathcal{M}(\mathbf{M}_{\mathrm{valid}})_{ij}
    =
    \begin{cases}
    0, & M_{\mathrm{valid},ij}=1,\\
    -\infty, & M_{\mathrm{valid},ij}=0.
    \end{cases}
\end{equation}
Thus invalid entries receive zero probability after the softmax. The final
undirected linkage matrix is obtained by symmetrizing $\mathbf{P}_t$:
\begin{equation}
    S_{ij,t}
    =
    \begin{cases}
    P_{ij,t}+P_{ji,t}, & i\neq j,\\
    0, & i=j.
    \end{cases}
    \label{eq:topo_sym}
\end{equation}

The cyclization type distribution is predicted as
\begin{equation}
    \mathbf{p}_{\tau,t}
    =
    \mathrm{Softmax}\left(
    \phi_{\mathrm{type}}\left(
    \mathrm{Pool}\left(\{\mathbf{h}_{ij,t}^{\mathrm{cyc}}\}_{i<j}\right)
    \right)
    \right).
\end{equation}

\noindent \textbf{Adaptive Topology Injection via Edge Message Modulation}
Given the predicted cyclization type distribution $\mathbf{p}_{\tau,t}$ and the
symmetric linkage matrix $\mathbf{S}_t$, we inject the online topology signal
into the AM-EGNN denoising network through edge-feature augmentation and
message modulation. The injection strength is controlled by timestep-dependent
schedules $\lambda_{\mathrm{feat}}(t)$, $\lambda_{\mathrm{bias}}(t)$, and
$\lambda_{\mathrm{gate}}(t)$. In our implementation, these schedules are linear
functions of the diffusion timestep and are used to gradually adjust the
topology-injection strength during denoising.

Since $\mathbf{S}_t$ has already been symmetrized in
Eq.~\eqref{eq:topo_sym}, we directly use $S_{ij,t}$ as the undirected linkage
probability between residues $i$ and $j$. We first fuse the linkage probability
and the cyclization type prediction into the edge features:
\begin{equation}
    \mathbf{e}_{ij,t}^{\mathrm{inj}}
    =
    \phi_{\mathrm{edge}}
    \left(
    \mathrm{Concat}
    \left(
    \mathbf{e}_{ij,t}^{\mathrm{base}},
    \mathbf{p}_{\tau,t},
    \lambda_{\mathrm{feat}}(t) S_{ij,t}
    \right)
    \right),
    \label{eq:edge_inj}
\end{equation}
where $\mathbf{e}_{ij,t}^{\mathrm{base}}$ denotes the base geometric RBF edge
feature at timestep $t$, and $\lambda_{\mathrm{feat}}(t) S_{ij,t}$ is treated
as a scalar topology feature.

To further modulate message passing, we compute an additive bias and a
multiplicative gate from the structural pair representation and the injected
topology signal:
\begin{align}
    \mathbf{b}_{ij,t}
    &=
    \phi_{\mathrm{bias}}
    \left(
    \mathbf{h}_{ij,t}^{\mathrm{struct}}
    \right)
    +
    \lambda_{\mathrm{bias}}(t) S_{ij,t}\mathbf{1}_{d_m},
    \\
    \mathbf{g}_{ij,t}
    &=
    \sigma
    \left(
    \phi_{\mathrm{gate}}
    \left(
    \mathbf{h}_{ij,t}^{\mathrm{struct}}
    \right)
    \right)
    \odot
    \left[
    \sigma
    \left(
    \lambda_{\mathrm{gate}}(t) S_{ij,t}
    \right)
    \mathbf{1}_{d_m}
    \right],
    \label{eq:gate_inj}
\end{align}
where $\mathbf{1}_{d_m}$ is an all-one vector with the same dimension as the
edge message, so that the scalar linkage probability is broadcast to the
message dimension. Here $\sigma(\cdot)$ denotes the sigmoid function and
$\odot$ denotes element-wise multiplication.

Within each AM-EGNN layer, let $\mathbf{m}_{ij,t}^{\mathrm{base}}$ denote the
base message from node $j$ to node $i$ computed using the injected edge feature
$\mathbf{e}_{ij,t}^{\mathrm{inj}}$. The topology-aware message is then defined
as
\begin{equation}
    \mathbf{m}_{ij,t}^{\mathrm{mod}}
    =
    \left(
    \mathbf{m}_{ij,t}^{\mathrm{base}}
    +
    \mathbf{b}_{ij,t}
    \right)
    \odot
    \mathbf{g}_{ij,t}.
    \label{eq:message_mod}
\end{equation}
This modulation allows the predicted cyclization topology to guide denoising
while preserving the flexibility of the AM-EGNN to refine local geometry through
the original geometric edge features.

\begin{proposition}[SE(3)-Equivariance]
\label{prop:equivariance}
Let $F_\theta$ denote the denoising update with topology injection.
Then $F_\theta$ is SE(3)-equivariant, i.e.,
\[
F_\theta(g \cdot \mathcal{Z}_t) = g \cdot F_\theta(\mathcal{Z}_t),
\quad \forall g \in \mathrm{SE}(3).
\]
\end{proposition}

During the late denoising phase, we switch to a deterministic strategy by forming $\mathbf{S}^{\mathrm{inj}}_t$ from the discretized linkage matrix $\mathbf{S}_t$. This "hard" linkage is treated as an explicit covalent edge, enhancing precise geometric closure in the final structure. Figure~\ref{fig:topology_injection} illustrates the online topology injection module.

\begin{figure}[t]
    \centering
    \vspace{-0.2em}
    \includegraphics[width=\columnwidth]{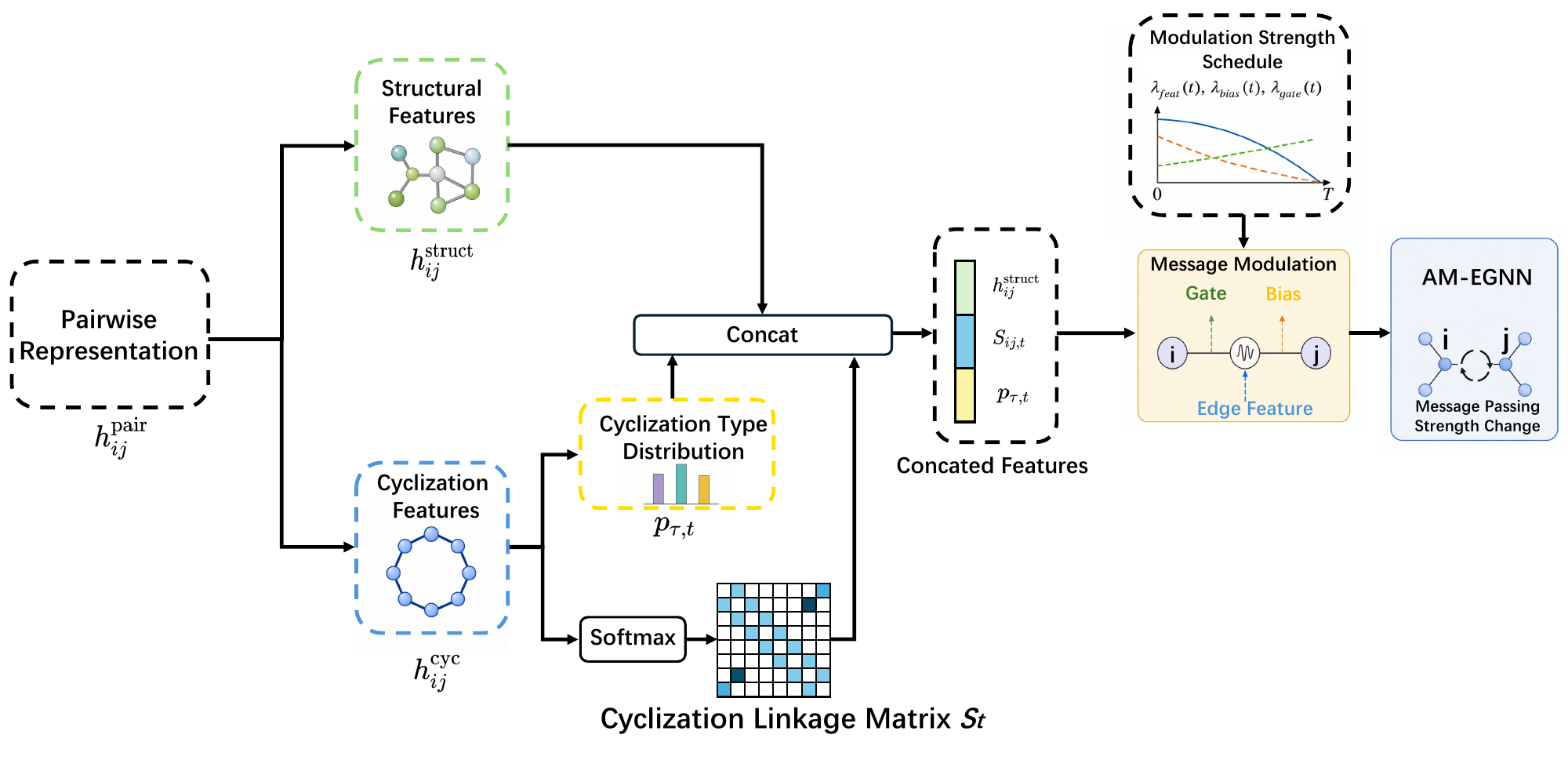}
    \caption{\textbf{Cyclization topology information injection.}
    APCyc predicts cyclization linkages and uses the topology signal to modulate AM-EGNN message passing.}
    \Description{Diagram of online topology injection where pairwise representations produce a linkage matrix and cyclization features that modulate message passing before the AM-EGNN update.}
    \label{fig:topology_injection}
    \vspace{-0.2em}
\end{figure}

\subsection{Geometry Latent Diffusion Process}
\label{sec:diffusion}

Following \citet{kong2024full}, we model cyclic peptide generation with a
DDPM-style latent diffusion process over residue-level latent variables. The
encoder outputs
$\mathcal{Z}_0=\{(z_{i,0},\tilde{z}_{i,0})\}_{i=1}^{L}$
are treated as clean latents, where $z_{i,0}$ is the invariant residue latent
and $\tilde{z}_{i,0}$ is the equivariant geometric latent.

Let $t\in\{0,\ldots,T\}$ be a discrete diffusion timestep. Given a
variance-preserving noise schedule $\{\beta_t\}_{t=1}^{T}$, we define
\[
    \eta_t=1-\beta_t,\quad
    \bar{\eta}_t=\prod_{s=1}^{t}\eta_s,\quad
    \alpha_t=\sqrt{\bar{\eta}_t},\quad
    \sigma_t=\sqrt{1-\bar{\eta}_t}.
\]
The forward perturbation admits the closed-form marginal
\begin{equation}
    z_{i,t}=\alpha_t z_{i,0}+\sigma_t\epsilon_i,
    \quad
    \tilde{z}_{i,t}
    =
    \alpha_t\tilde{z}_{i,0}
    +
    \sigma_t\tilde{\epsilon}_i,
    \label{eq:ddpm_componentwise_marginal}
\end{equation}
where $\epsilon_i$ and $\tilde{\epsilon}_i$ are independent standard Gaussian
noise terms with shapes matching $z_{i,0}$ and $\tilde{z}_{i,0}$, respectively.

The denoising network
$\boldsymbol{\epsilon}_{\theta}(\mathcal{Z}_t,\mathcal{C},t)$ predicts the
noise added in Eq.~\eqref{eq:ddpm_componentwise_marginal}. It is implemented as
the cyclization-aware AM-EGNN described in Sec.~\ref{sec:cyc_injection}. During
generation, we start from
$\mathcal{Z}_T\sim\mathcal{N}(\mathbf{0},\mathbf{I})$ and apply the standard
DDPM reverse update
\begin{equation}
    \mathcal{Z}_{t-1}
    =
    \frac{1}{\sqrt{\eta_t}}
    \left(
    \mathcal{Z}_t
    -
    \frac{\beta_t}{\sigma_t}
    \boldsymbol{\epsilon}_{\theta}(\mathcal{Z}_t,\mathcal{C},t)
    \right)
    +
    \sqrt{\tilde{\beta}_t}\boldsymbol{\xi},
    \quad
    \boldsymbol{\xi}\sim\mathcal{N}(\mathbf{0},\mathbf{I}),
    \label{eq:ddpm_sampling_step}
\end{equation}
where
$\tilde{\beta}_t=\frac{1-\bar{\eta}_{t-1}}{1-\bar{\eta}_t}\beta_t$, and
$\boldsymbol{\xi}$ is set to zero when $t=1$.

\textbf{Training Objective}
At training time, we uniformly sample a timestep $t \in \{1,\ldots,T\}$, construct $\mathcal{Z}_t$ using Eq.~\eqref{eq:ddpm_componentwise_marginal}, and optimize the DDPM noise-prediction loss
\begin{equation}
    \mathcal{L}_{\mathrm{diff}}
    =
    \mathbb{E}_{t,\mathcal{Z}_0,\boldsymbol{\epsilon}}
    \left[
    \left\|
    \boldsymbol{\epsilon}
    -
    \boldsymbol{\epsilon}_{\theta}(\mathcal{Z}_t,\mathcal{C},t)
    \right\|_2^2
    \right],
    \label{eq:diff_loss}
\end{equation}
where $\boldsymbol{\epsilon}=\{(\epsilon_i,\tilde{\epsilon}_i)\}_{i=1}^{L}$.

We further supervise the online topology prediction using the pre-symmetrized
linkage probability matrix $\mathbf{P}_t$ from Eq.~\eqref{eq:topo_softmax} and
the cyclization type distribution $\mathbf{p}_{\tau,t}$. Let
$\mathbf{S}^*$ denote the ground-truth linkage matrix and $\tau^*$ denote the
ground-truth cyclization type. We define
\[
    Y_{ij}^{*}
    =
    \frac{
    S_{ij}^{*}M_{\mathrm{valid},ij}
    }{
    \sum_{u,v}S_{uv}^{*}M_{\mathrm{valid},uv}
    },
\]
and use
\begin{equation}
    \mathcal{L}_{\mathrm{topo}}
    =
    \mathbb{E}_{t,\mathcal{Z}_0,\boldsymbol{\epsilon}}
    \left[
    -
    \sum_{i,j}
    Y_{ij}^{*}\log(P_{ij,t})
    +
    \lambda_{\mathrm{type}}
    \mathrm{CE}(\mathbf{p}_{\tau,t},\tau^*)
    \right].
    \label{eq:topo_loss}
\end{equation}
The final objective is
\begin{equation}
    \mathcal{L}_{\mathrm{total}}
    =
    \mathcal{L}_{\mathrm{diff}}
    +
    \lambda_{\mathrm{topo}}\mathcal{L}_{\mathrm{topo}}.
    \label{eq:total_loss}
\end{equation}

\subsection{Property Optimization via Bayesian Posterior Guidance}
\label{sec:guidance}

To steer generation toward peptide candidates satisfying desired therapeutic
criteria, such as binding affinity, permeability, solubility, or protease
resistance, we perform property-guided sampling in the latent space. Given the
target property vector
$\mathbf{Y}=(y_1,\ldots,y_M)\in\mathbb{R}^M$
and receptor context $\mathcal{C}$, the goal is to sample from the
property-conditioned posterior
$p_t(\mathcal{Z}_t\mid\mathcal{C},\mathbf{Y})$ rather than only from the
receptor-conditioned prior $p_t(\mathcal{Z}_t\mid\mathcal{C})$.

By Bayes' rule, the posterior score can be decomposed as
\begin{equation}
    \nabla_{\mathcal{Z}_t}
    \log p_t(\mathcal{Z}_t\mid\mathcal{C},\mathbf{Y})
    =
    \nabla_{\mathcal{Z}_t}
    \log p_t(\mathcal{Z}_t\mid\mathcal{C})
    +
    \nabla_{\mathcal{Z}_t}
    \log p_t(\mathbf{Y}\mid\mathcal{Z}_t,\mathcal{C}).
    \label{eq:posterior_score}
\end{equation}
The first term is represented by the DDPM denoising network through
$s_\theta(\mathcal{Z}_t,\mathcal{C},t)
\approx
-\boldsymbol{\epsilon}_{\theta}(\mathcal{Z}_t,\mathcal{C},t)/\sigma_t$.
The second term is intractable because drug-relevant properties are not directly
defined on noisy latent states. We therefore introduce a time-conditioned
differentiable surrogate $\Psi_{\psi}$ and define the property energy
\begin{equation}
    \mathcal{E}_{\psi}
    (\mathcal{Z}_t,t;\mathbf{Y},\mathcal{C})
    =
    \sum_{m=1}^{M}
    \lambda_m
    \ell_m
    \left(
    \hat{y}_{m,t},
    y_m
    \right),
    \quad
    \hat{\mathbf{Y}}_t
    =
    \Psi_{\psi}(\mathcal{Z}_t,\mathcal{C},t),
    \label{eq:property_energy}
\end{equation}
where $\hat{\mathbf{Y}}_t=(\hat{y}_{1,t},\ldots,\hat{y}_{M,t})$ is the predicted
property vector at timestep $t$, $\ell_m$ measures the deviation from the
$m$-th desired property target, and $\lambda_m$ controls its relative
importance. This energy serves as a differentiable proxy for the negative
log-likelihood
$-\log p_t(\mathbf{Y}\mid\mathcal{Z}_t,\mathcal{C})$. The resulting guided posterior is written as
\begin{equation}
    p_t(\mathcal{Z}_t\mid\mathcal{C},\mathbf{Y})
    \propto
    p_t(\mathcal{Z}_t\mid\mathcal{C})
    \exp
    \left(
    -\gamma_t
    \mathcal{E}_{\psi}
    (\mathcal{Z}_t,t;\mathbf{Y},\mathcal{C})
    \right),
    \label{eq:guided_posterior}
\end{equation}
where $\gamma_t$ is a timestep-dependent guidance scale. Therefore, the guided
score is approximated by
\begin{equation}
    s_{\theta,\psi}^{\mathrm{guide}}
    (\mathcal{Z}_t,\mathcal{C},\mathbf{Y},t)
    =
    -
    \frac{
    \boldsymbol{\epsilon}_{\theta}(\mathcal{Z}_t,\mathcal{C},t)
    }{\sigma_t}
    -
    \gamma_t
    \nabla_{\mathcal{Z}_t}
    \mathcal{E}_{\psi}
    (\mathcal{Z}_t,t;\mathbf{Y},\mathcal{C}).
    \label{eq:guided_score}
\end{equation}
Equivalently, under the DDPM noise-prediction parameterization, we use the
guided noise estimate
\begin{equation}
    \boldsymbol{\epsilon}_{\theta,\psi}^{\mathrm{guide}}
    =
    \boldsymbol{\epsilon}_{\theta}
    +
    \sigma_t\gamma_t
    \nabla_{\mathcal{Z}_t}
    \mathcal{E}_{\psi}
    (\mathcal{Z}_t,t;\mathbf{Y},\mathcal{C}).
    \label{eq:guided_noise}
\end{equation}
During sampling, we replace $\boldsymbol{\epsilon}_{\theta}$ with
$\boldsymbol{\epsilon}_{\theta,\psi}^{\mathrm{guide}}$ in the DDPM reverse
update:
\begin{equation}
    \mathcal{Z}_{t-1}
    =
    \frac{1}{\sqrt{\eta_t}}
    \left(
    \mathcal{Z}_t
    -
    \frac{\beta_t}{\sigma_t}
    \boldsymbol{\epsilon}_{\theta,\psi}^{\mathrm{guide}}
    \right)
    +
    \sqrt{\tilde{\beta}_t}\boldsymbol{\xi}.
    \label{eq:guided_ddpm_update}
\end{equation}
Thus, the energy gradient jointly guides the invariant residue latents
$\{z_{i,t}\}_{i=1}^{L}$ and the equivariant geometric latents
$\{\tilde{z}_{i,t}\}_{i=1}^{L}$, optimizing property with structural
refinement during generation.

\begin{figure}[t]
    \centering
    \includegraphics[width=\columnwidth]{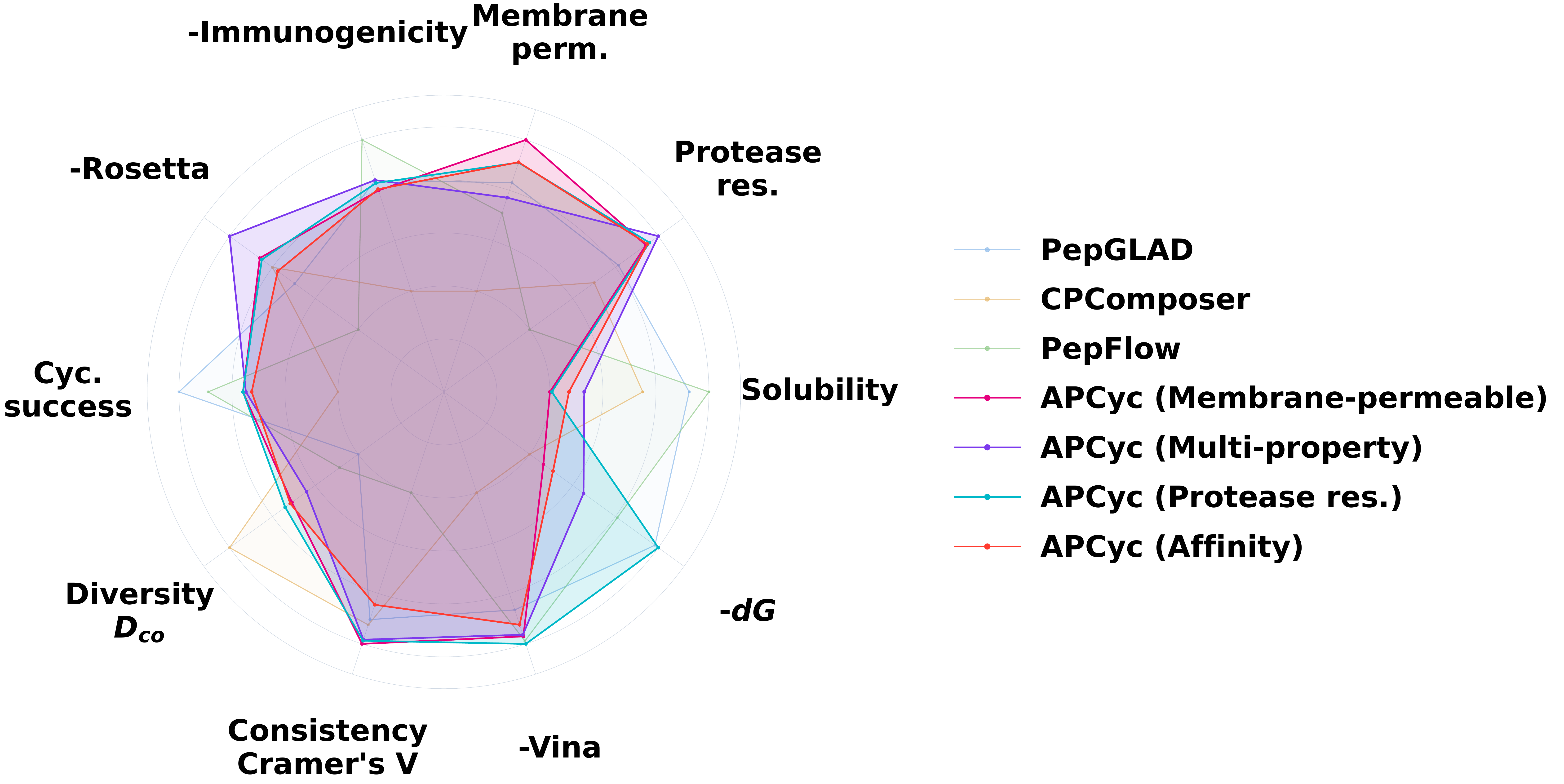}
    \caption{\textbf{Overall comparison across different methods on therapeutic properties and generative metrics.}}
    \Description{Radar chart comparing PepGLAD, CPComposer, PepFlow, and APCyc guidance variants across solubility, protease resistance, membrane permeability, immunogenicity, binding affinity, cyclization success, diversity, and consistency metrics.}
    \label{fig:property_radar}
    \vspace{-0.25em}
\end{figure}

\begin{algorithm}[t]
\caption{APCyc Inference with Topology and Property Guidance}
\label{alg:inference}
\begin{algorithmic}[1]
\STATE \textbf{Input:}
Denoiser $\boldsymbol{\epsilon}_{\theta}$, decoder $\mathcal{D}_{\xi}$,
receptor context $\mathcal{C}$, property targets
$\mathbf{Y}=(y_1,\ldots,y_M)$, property surrogate $\Psi_{\psi}$,
base weights $\mathbf{w}=\{w_m\}_{m=1}^{M}$, validity mask
$\mathbf{M}_{\mathrm{valid}}$, aggregation operator $\mathrm{Agg}(\cdot)$,
norm operator $\mathrm{Norm}(\cdot)$, guidance scale $\gamma_t$, hard-linkage
threshold $t_{\mathrm{hard}}$

\STATE \textbf{Initialize:}
Sample
$\mathcal{Z}_T=\{(z_{i,T},\tilde{z}_{i,T})\}_{i=1}^{L}
\sim \mathcal{N}(\mathbf{0},\mathbf{I})$

\FOR{$t=T,\ldots,1$}

    \STATE
    $\mathbf{p}_{\tau,t},\mathbf{P}_t,\mathbf{S}_t
    \gets
    \mathrm{CyclizationHeads}
    (\mathcal{Z}_t,\mathcal{C},t;\mathbf{M}_{\mathrm{valid}})$

    \STATE
    $\mathbf{S}^{\mathrm{inj}}_t
    \gets
    \begin{cases}
        \mathrm{SymOneHot}
        \left(
        \arg\max_{i<j:\,M_{\mathrm{valid},ij}=1}
        S_{ij,t}
        \right),
        & t\le t_{\mathrm{hard}},\\
        \mathbf{S}_t,
        & t>t_{\mathrm{hard}}.
    \end{cases}$

    \STATE
    $\hat{\mathbf{Y}}_t
    \gets
    \Psi_{\psi}(\mathcal{Z}_t,\mathcal{C},t)$

    \FOR{$m=1,\ldots,M$}
        \STATE
        $\hat{E}_{m,t}
        \gets
        \ell_m(\hat{y}_{m,t},y_m)$
        \STATE
        $\mathbf{g}_{m,t}
        \gets
        \nabla_{\mathcal{Z}_t}\hat{E}_{m,t}$
        \STATE
        $N_{m,t}
        \gets
        \mathrm{Norm}(\mathbf{g}_{m,t})$
    \ENDFOR

    \STATE
    $\lambda_m(t)
    \gets
    \mathrm{stopgrad}
    \left(
    w_m
    \cdot
    \frac{
    \mathrm{Agg}(\{N_{j,t}\}_{j=1}^{M})
    }{
    N_{m,t}+\varepsilon_{\mathrm{bal}}
    }
    \right),
    \quad m=1,\ldots,M$

    \STATE
    $\hat{E}_{\mathrm{bal}}(\mathcal{Z}_t,t)
    \gets
    \sum_{m=1}^{M}
    \lambda_m(t)\hat{E}_{m,t}$

    \STATE
    $\mathbf{g}_{\mathrm{prop},t}
    \gets
    \nabla_{\mathcal{Z}_t}
    \hat{E}_{\mathrm{bal}}(\mathcal{Z}_t,t)$

    \STATE
    $\hat{\boldsymbol{\epsilon}}_{\theta}
    \gets
    \boldsymbol{\epsilon}_{\theta}
    (\mathcal{Z}_t,\mathcal{C},t;
    \mathbf{S}^{\mathrm{inj}}_t,\mathbf{p}_{\tau,t})$

    \STATE
    $\hat{\boldsymbol{\epsilon}}_{\theta,\psi}^{\mathrm{guide}}
    \gets
    \hat{\boldsymbol{\epsilon}}_{\theta}
    +
    \gamma_t\sigma_t
    \mathbf{g}_{\mathrm{prop},t}$

    \STATE
    $\mathcal{Z}_{t-1}
    \gets
    \mathrm{DDPMReverseStep}
    \left(
    \mathcal{Z}_t,
    \hat{\boldsymbol{\epsilon}}_{\theta,\psi}^{\mathrm{guide}},
    t
    \right)$

\ENDFOR

\STATE
$\mathbf{p}_{\tau,0},\mathbf{P}_0,\mathbf{S}_0
\gets
\mathrm{CyclizationHeads}
(\mathcal{Z}_0,\mathcal{C},0;\mathbf{M}_{\mathrm{valid}})$

\STATE
$\hat{\mathcal{P}}
\gets
\mathcal{D}_{\xi}(\mathcal{Z}_0)$

\STATE
\textbf{return}
$\hat{\mathcal{P}},
\arg\max \mathbf{p}_{\tau,0},
\mathbf{S}_0$

\end{algorithmic}
\end{algorithm}

\subsection{Training and Inference}
\label{sec:training}

APCyc is trained in a decoupled two-stage paradigm. We first train a joint
latent autoencoder to learn residue-level invariant and equivariant latents,
and then train a DDPM-style latent diffusion model over the frozen latent space.

\noindent \textbf{Stage 1: Joint Latent Autoencoding}
We train the encoder $\mathcal{E}_{\phi}$ and decoder $\mathcal{D}_{\xi}$ to
map a peptide graph $\mathcal{P}$ into clean residue-level latents
$\mathcal{Z}_0=\{(z_{i,0},\tilde{z}_{i,0})\}_{i=1}^{L}$ and reconstruct the
full-atom peptide. The autoencoding objective is
\begin{equation}
    \mathcal{L}_{\mathrm{AE}}
    =
    \frac{1}{L}
    \sum_{i=1}^{L}
    \left(
    \mathcal{L}_{\mathrm{rec}}(i)
    +
    \mathcal{L}_{\mathrm{latent}}(i)
    \right).
    \label{eq:ae_loss}
\end{equation}
The reconstruction term is
\begin{equation}
    \mathcal{L}_{\mathrm{rec}}(i)
    =
    \mathrm{CE}(s_i,\hat{s}_i)
    +
    \|\mathbf{X}_i-\hat{\mathbf{X}}_i\|_F^2
    +
    \gamma_{\mathrm{phys}}
    \mathcal{L}_{\mathrm{phys}}(i),
    \label{eq:rec_loss}
\end{equation}
where $s_i$ and $\hat{s}_i$ denote the ground-truth and reconstructed residue
types, $\mathbf{X}_i$ and $\hat{\mathbf{X}}_i$ denote the ground-truth and
reconstructed full-atom coordinates, and $\mathcal{L}_{\mathrm{phys}}$ enforces
valid local peptide geometry, such as bond lengths and bond angles. The latent
regularization term $\mathcal{L}_{\mathrm{latent}}$ regularizes the invariant
latent $z_{i,0}$ and the equivariant geometric latent $\tilde{z}_{i,0}$; the
detailed autoencoder objective is provided in Appendix~\ref{vae}.

\noindent \textbf{Stage 2: Latent Diffusion Training}
In the second stage, the autoencoder is frozen and the denoising network
$\boldsymbol{\epsilon}_{\theta}$ is optimized using the joint objective
$\mathcal{L}_{\mathrm{total}}$ in Sec.~\ref{sec:diffusion}. This stage uses the
DDPM noise-prediction objective over noisy latents $\mathcal{Z}_t$, together
with auxiliary topology supervision for the Online Topology Injection module in
Sec.~\ref{sec:cyc_injection}. Thus, the denoiser learns both latent geometry
denoising and timestep-dependent cyclization topology prediction.

\noindent \textbf{Inference with Adaptive Balanced Guidance}
For property-guided generation, we use the target property vector
$\mathbf{Y}=(y_1,\ldots,y_M)$ defined in Sec.~\ref{sec:preliminaries}. At each reverse
diffusion step $t$, the property surrogate $\Psi_{\psi}$ predicts
\begin{equation}
    \hat{\mathbf{Y}}_t
    =
    \Psi_{\psi}(\mathcal{Z}_t,\mathcal{C},t),
    \quad
    \hat{\mathbf{Y}}_t
    =
    (\hat{y}_{1,t},\ldots,\hat{y}_{M,t}).
\end{equation}
For each property target, we define a property-wise energy
\begin{equation}
    \hat{E}_{m,t}
    =
    \ell_m(\hat{y}_{m,t},y_m),
    \quad
    m=1,\ldots,M,
    \label{eq:property_wise_energy}
\end{equation}
where $\ell_m$ measures the deviation from the desired $m$-th property target.

To balance multiple objectives, we compute the gradient norm of each
property-specific energy:
\begin{equation}
    \mathbf{g}_{m,t}
    =
    \nabla_{\mathcal{Z}_t}\hat{E}_{m,t},
    \quad
    N_{m,t}
    =
    \mathrm{Norm}(\mathbf{g}_{m,t}).
    \label{eq:property_grad_norm}
\end{equation}
The adaptive weight for the $m$-th property is
\begin{equation}
    \lambda_m(t)
    =
    \mathrm{stopgrad}
    \left(
    w_m
    \cdot
    \frac{
    \mathrm{Agg}(\{N_{j,t}\}_{j=1}^{M})
    }{
    N_{m,t}+\varepsilon_{\mathrm{bal}}
    }
    \right),
    \label{eq:adaptive_weight}
\end{equation}
where $w_m$ is the base importance weight, $\mathrm{Agg}(\cdot)$ is an
aggregation operator over gradient norms, and $\varepsilon_{\mathrm{bal}}$ is a
small constant for numerical stability. The balanced property energy and its
guidance gradient are then
\begin{equation}
    \hat{E}_{\mathrm{bal}}(\mathcal{Z}_t,t)
    =
    \sum_{m=1}^{M}
    \lambda_m(t)\hat{E}_{m,t},
    \quad
    \mathbf{g}_{\mathrm{prop},t}
    =
    \nabla_{\mathcal{Z}_t}
    \hat{E}_{\mathrm{bal}}(\mathcal{Z}_t,t).
    \label{eq:balanced_energy}
\end{equation}
Following Sec.~\ref{sec:guidance}, we form the guided noise estimate as
\begin{equation}
    \boldsymbol{\epsilon}_{\theta,\psi}^{\mathrm{guide}}
    =
    \boldsymbol{\epsilon}_{\theta}
    +
    \gamma_t\sigma_t
    \mathbf{g}_{\mathrm{prop},t},
    \label{eq:balanced_guided_noise}
\end{equation}
where $\gamma_t$ is the timestep-dependent guidance scale. The guided noise is
then used in the DDPM reverse update in Eq.~\eqref{eq:guided_ddpm_update}.
The complete inference procedure is summarized in Algorithm~\ref{alg:inference}.

\begin{table*}[t]
\centering
\caption{Property metrics (mean over valid samples). Higher is better for Solubility, Protease resistance, and membrane permeability (Pproxy); lower is better for Immunogenicity. Top-1 and Top-2 are highlighted by \textbf{bold} and \underline{underline}, respectively.}
\label{tab:property_results}
\vspace{0.35em}
\resizebox{0.82\textwidth}{!}{%
\begin{tabular}{lcccc}
\toprule
\textbf{Method} & \textbf{Solubility} $\uparrow$ & \textbf{Protease res.} $\uparrow$ & \textbf{Membrane perm.} $\uparrow$ & \textbf{Immunogenicity} $\downarrow$ \\
\midrule

\rowcolor{gray!10}
\multicolumn{5}{l}{\textbf{Baselines}} \\
PepGLAD                & \underline{1.593} & -1.555 & -0.010 & -9.097 \\
CPComposer             & 1.520 & -1.604 & -0.307 & -7.576 \\
PepFlow                & \textbf{1.624} & -1.735 & -0.093 & \textbf{-9.654} \\

\midrule
\\[-0.6em]

\rowcolor{gray!10}
\multicolumn{5}{l}{\textbf{Ours}} \\
Membrane-permeable     & 1.374 & -1.499 & \textbf{0.107} & -8.960 \\
Multi-property (joint) & 1.428 & \textbf{-1.474} & -0.051 & \underline{-9.102} \\
Protease res.-optimized    & 1.377 & \underline{-1.492} & 0.045 & -9.062 \\
Affinity-optimized     & 1.404 & -1.496 & \underline{0.046} & -8.977 \\
\bottomrule
\end{tabular}
}
\end{table*}

\section{Experiment}
\subsection{Experimental Setup}
\textbf{Dataset} In this work, we utilize \textbf{CPCore}, a high-quality refined subset of the \textbf{CPSea} dataset \citep{yang2025cpsea}. CPSea is a large-scale collection of 2.71 million cyclic peptide--protein complexes curated from the AlphaFold Database (AFDB). \textbf{CPCore} is constructed through intersectional filtering to ensure biophysical fidelity and essential properties. To be specific, the dataset consists of \textbf{71,867} unique complexes, a scale suitable for end-to-end training of generative models from scratch. It covers diverse and classical cyclization topologies, including \textbf{mainchain} amide bonds, \textbf{disulfide} bonds, and side-chain \textbf{isopeptide} (specifically Lys-Asp/Asn) linkages.
Following \citet{yang2025cpsea}, we partition the dataset based on FoldSeek structural clusters to prevent data leakage. Of the 30,819 clusters in CPCore, we randomly reserve 616 clusters ($\sim$ 2\%)
for validation and use the remainder for training. For evaluation, we employ the Large Non-Redundant (LNR) dataset. After filtering for sequence identity ($<$40\% via MMseqs2) 
and structural suitability, the final test set comprises 56 unique targets. Concretely, after the cluster-based split, the CPCore training set contains 70,284 complexes and the validation set contains 1,583 complexes.

\noindent \textbf{Property Collection}
To facilitate multi-objective optimization and ensure therapeutic viability, we developed a comprehensive property evaluation pipeline. This pipeline integrates a diverse suite of computational tools to characterize the physicochemical, pharmacological, and binding profiles of the cyclic peptides. The following provides a detailed overview of the specialized property prediction tools, and we provide implementation details for property collection in Appendix~\ref{metrics}.
\begin{itemize}
    \item \textbf{Physicochemical Descriptors}: Hydrophobicity is assessed via the Grand Average of Hydropathy (GRAVY) using the Kyte-Doolittle scale \citep{kyte1982simple}. Additional rational descriptors, including partition coefficient (logP) \citep{wildman1999prediction}, topological polar surface area (TPSA) \citep{ertl2000fast}, and atom-normalized TPSA (rTPSA)
    \citep{ertl2000fast} are calculated via RDKit.
    \item \textbf{Solubility}: Intrinsic solubility is predicted via the CamSol method \citep{sormanni2015camsol}.
    \item \textbf{Protease resistance}: Proteolytic stability is assessed via ProsperousPlus
    \citep{li2023prosperousplus} to estimate the susceptibility of the cyclic backbone to enzymatic cleavage.
    \item \textbf{Immunogenicity}: Potential binding and presentation to MHC molecules are predicted using BigMHC \citep{albert2023deep}.
    \item \textbf{Binding Affinity}: Interface binding energetics are characterized by AutoDock Vina score \citep{trott2010autodock} and Rosetta interface energy (\textit{rosetta\_dG}) \citep{alford2017rosetta} using the interface\_analyze protocol.
\end{itemize}

\noindent \textbf{Baselines.} We compare our method with several representative state-of-the-art peptide generative models. PepGLAD \citep{kong2024full} performs target-aware full-atom peptide design using geometric latent diffusion with a variational autoencoder for joint sequence–structure generation. Building upon the PepGLAD architecture, CP-Composer \citep{jiang2025cpcomposer} enables zero-shot cyclic peptide generation by decomposing cyclization into composable geometric constraints. Since CP-Composer focuses on the transferability from linear to cyclic peptides, we directly employ its original linear peptide generation weights and apply head-to-tail cyclization constraints for our comparisons. Additionally, we evaluate PepFlow \citep{li2024full}, a multi-modal conditional flow-matching model that jointly models residue orientation, torsion angles, and sequence for full-atom co-design. To ensure a fair and rigorous comparison, PepGLAD and PepFlow were retrained from scratch on the CPCore dataset using identical data splits and training protocols, while CP-Composer was implemented in its intended zero-shot capacity.
\begin{table*}[t]
\centering
\caption{Main results (mean over valid samples). Top-1 and Top-2 are highlighted by \textbf{bold} and \underline{underline}, respectively.}
\label{tab:main_results}
\vspace{0.35em}
\resizebox{0.90\textwidth}{!}{%
\begin{tabular}{l c cc c c cc}
\toprule
\textbf{Method} &
\multicolumn{1}{c}{\textbf{Stability}} &
\multicolumn{2}{c}{\textbf{Success}} &
\multicolumn{1}{c}{\textbf{Diversity}} &
\multicolumn{1}{c}{\textbf{Consistency}} &
\multicolumn{2}{c}{\textbf{Affinity}} \\
\cmidrule(lr){2-2}\cmidrule(lr){3-4}\cmidrule(lr){5-5}\cmidrule(lr){6-6}\cmidrule(lr){7-8}
& \textbf{Rosetta} $\downarrow$ & \textbf{Cyc.} $\uparrow$ & \textbf{Energy} $\uparrow$ & \textbf{$D_{co}$} $\uparrow$ & \textbf{Cramér's V} $\uparrow$ & \textbf{Vina} $\downarrow$ & \textbf{$dG$} $\downarrow$ \\
\midrule

\rowcolor{gray!10}
\multicolumn{8}{l}{\textbf{Baselines}} \\
PepGLAD     & -735.528 & \textbf{0.973} & \textbf{0.883} & 0.669 & 0.948 & -4.544 & \underline{-12.140} \\
CPComposer  & -743.422 & 0.809 & 0.700 & \textbf{0.898} & 0.953 & -3.686 & -10.716 \\
PepFlow     & -713.165 & \underline{0.943} & \underline{0.711} & 0.702 & 0.828 & \underline{-4.772} & -11.711 \\

\midrule
\\[-0.6em]

\rowcolor{gray!10}
\multicolumn{8}{l}{\textbf{Ours}} \\
Membrane-permeable & \underline{-747.898} & 0.907 & 0.671 & 0.787 & \textbf{0.971} & -4.738 & -10.872 \\
Multi-property (joint) & \textbf{-758.545} & 0.904 & 0.664 & 0.761 & 0.967 & -4.726 & -11.328 \\
Protease res.-optimized & -747.176 & 0.907 & 0.689 & \underline{0.799} & \underline{0.968} & \textbf{-4.793} & \textbf{-12.178} \\
Affinity-optimized & -741.560 & 0.898 & 0.648 & 0.790 & 0.934 & -4.653 & -10.981 \\
\bottomrule
\end{tabular}
}
\end{table*}

\noindent \textbf{Metrics} To evaluate the efficacy of our method and the therapeutic potential of the designed cyclic peptides, we first assess their pharmacological viability and structural rationality. Following the property evaluation pipeline described previously, we calculate a suite of \textbf{drug-like properties}, including physicochemical descriptors (LogP, TPSA, GRAVY), intrinsic solubility, protease resistance and immunogenicity. We define \textbf{Affinity} using the AutoDock Vina score \citep{trott2010autodock}, which measures empirical docking fitness, and the Rosetta interface energy ($dG$) \citep{alford2017rosetta}, which quantifies the binding free energy between the peptide and its target.

Furthermore, we evaluate the generative performance using success rates and ensemble-level distributional metrics. 
A generated sample is considered \textbf{Cyclization Success} if the distance between the cyclization residues falls within the range of $[3.0, 8.0]$~\AA, 
while \textbf{Energy Success} is defined by a negative interface binding energy ($dG < 0$). 
In addition to these criteria, we adopt the Rosetta total score~\citep{chaudhury2010pyrosetta} as a measure of \textbf{Stability}, 
reflecting the physical plausibility and energetic favorability of the peptide’s three-dimensional binding conformation. 
To quantify \textbf{Diversity}, we utilize the co-diversity metric 
$D_{co} = \sqrt{D_{seq} \times D_{struct}}$ \citep{kong2024full}, 
where $D_{seq}$ and $D_{struct}$ denote the fractions of clusters over valid samples obtained via single-linkage hierarchical clustering 
based on sequence distance ($1 -$ similarity, threshold $0.4$) and $C_\alpha$ RMSD (threshold $4.0$~\AA), respectively. 
The correspondence between sequence and structure spaces is further evaluated by \textbf{Consistency}, measured via Cramér’s V~\citep{cramer1999mathematical}, 
which captures the correlation between sequence and structural cluster assignments. The detailed metrics are in Appendix~\ref{metrics}.

\subsection{Main Results}
For each receptor in test set, we generate $10$ candidate peptides. 
Samples that pass both the cyclization and energy criteria are regarded as \textbf{valid}. 
For each metric, we first compute the mean over valid samples within the same receptor, 
and then average across receptors to obtain the final reported performance.

Since our method enables demand-driven cyclic peptide design, we report several representative therapeutic design objectives, including affinity-optimized peptides, protease-resistant peptides, highly membrane-permeable peptides, and multi-property jointly optimized peptides. We provide the detailed guidance proportion in Appendix~\ref{detail}. As shown in Table~\ref{tab:property_results}, our method allows explicit control over multiple key physicochemical objectives, including solubility, protease resistance, membrane permeability, and immunogenicity.
Property-specific optimization successfully steers generated peptides toward the intended design targets. 
Figure~\ref{fig:property_radar} visualizes these normalized cross-metric trends, showing that APCyc variants occupy complementary regions of the therapeutic--generative trade-off space while improving their intended guidance targets.
For instance, the membrane-permeable setting achieves the highest permeability proxy among all methods while maintaining comparable solubility and immunogenicity, indicating that permeability enhancement does not substantially compromise other safety-related attributes. 
Meanwhile, the multi-property joint optimization yields the best protease resistance and competitive immunogenicity, demonstrating the capability of our framework to balance trade-offs across multiple therapeutic constraints rather than optimizing a single metric in isolation. 
Protease resistance- and affinity-oriented variants further preserve favorable physicochemical profiles, suggesting that targeted optimization remains compatible with overall molecular feasibility.

Notably, we also observe a mild decrease in solubility under guided optimization. This trend likely arises from the intrinsic physicochemical coupling between membrane permeability, hydrophobicity, and structural compactness, where promoting permeability oriented features can partially counteract solubility.
Such behavior is consistent with known trade-offs in peptide drug design and further highlights the necessity of controllable multi-objective optimization~\citep{ramelot2023cell}.
Together, these results verify that our framework enables controllable and therapeutically meaningful modulation of peptide properties beyond baseline generative behavior.

\begin{table}[t]
\centering
\caption{Ablation on protease resistance and immunogenicity.}
\label{tab:ablation_protease_immuno}
\begin{tabular}{lcc}
\toprule
\textbf{Method} & \textbf{Protease res.} $\uparrow$ & \textbf{Immunogenicity} $\downarrow$ \\
\midrule
Multi-property (guided) & \textbf{-1.474} & \textbf{-9.102} \\
Protease res.-optimized     & -1.492 &  -9.062 \\
Base (no guidance)      & -1.519 & -9.011 \\
\bottomrule
\end{tabular}
\end{table}

Beyond property-level improvements, Table~\ref{tab:main_results} shows that the generated peptides also retain strong performance under standard generative evaluation criteria. Specifically, our guided variants achieve consistently strong \textbf{Stability}, outperforming the baselines in most cases, with \textit{Affinity-optimized} being the only exception that is slightly weaker than CPComposer.
In addition, the \textit{Membrane-permeable} variant attains the highest \textbf{Consistency} score, indicating improved sequence--structure agreement under permeability-oriented guidance. 
Meanwhile, the \textit{Protease res.-optimized} variant achieves the best binding affinity among our models, yielding the most favorable \textbf{Vina} score and interface free energy ($dG$). 
Finally, our method maintains competitive \textbf{Diversity} across guidance settings, demonstrating that improved stability and affinity do not come at the expense of ensemble-level variation. This suggests that enhanced therapeutic property control remains compatible with generative quality and binding effectiveness.
Overall, the proposed framework achieves balanced improvements across both functional therapeutic objectives and conventional generative metrics, supporting its suitability for practical cyclic peptide design.

\subsection{Target-level Case Study}
\label{sec:case_study}

Beyond the aggregate benchmark results, we further conduct target-level case studies to examine whether APCyc can generate cyclic peptides with favorable binding and multi-property profiles on individual targets. Specifically, we select two representative targets, 3rc4 and 4xal, and compare APCyc-generated cyclic peptides with baseline references on the same targets.

As shown in Table~\ref{tab:case_study}, APCyc consistently achieves stronger binding scores and more favorable multi-property profiles than the PepFlow/PepGLAD mean on both targets. For Vina docking score and immunogenicity, where lower values are preferred, APCyc obtains substantially lower scores. For membrane permeability and protease resistance, where higher values are preferred, APCyc also achieves better or comparable results. In particular, APCyc improves membrane permeability from $0.806$ to $2.408$ on 3rc4, and from $0.917$ to $3.015$ on 4xal. These target-level results suggest that the overall benchmark improvements of APCyc are reflected not only in averaged metrics, but also in representative individual cases.

\begin{table}[t]
\centering
\caption{
Target-level case study on two representative targets. Baseline mean denotes the PepFlow/PepGLAD mean.
}
\label{tab:case_study}
\vspace{0.35em}
\renewcommand{\arraystretch}{0.92}
\setlength{\tabcolsep}{2.0pt}
\begin{tabular}{@{}llrrrr@{}}
\toprule
\textbf{Target} & \textbf{Method} & \textbf{Vina} $\downarrow$ & \textbf{Imm.} $\downarrow$ & \textbf{Perm.} $\uparrow$ & \textbf{Prot.} $\uparrow$ \\
\midrule
3rc4 & APCyc & \textbf{-4.960} & \textbf{-12.615} & \textbf{2.408} & \textbf{-1.637} \\
3rc4 & Baseline mean & -1.149 & -10.636 & 0.806 & -1.639 \\
\midrule
4xal & APCyc & \textbf{-6.820} & \textbf{-12.360} & \textbf{3.015} & \textbf{-1.256} \\
4xal & Baseline mean & -4.023 & -10.688 & 0.917 & -1.768 \\
\bottomrule
\end{tabular}
\renewcommand{\arraystretch}{1}
\setlength{\tabcolsep}{6pt}
\end{table}

\subsection{Ablation}
To evaluate the contribution of the proposed guidance mechanism, we conduct an ablation study by comparing the full APCyc model with a variant where the guidance component is removed. We further analyze the impact of guidance on key therapeutic properties through targeted ablations. 
As shown in Table~\ref{tab:ablation_protease_immuno}, guided optimization achieves competitive or improved protease resistance and immunogenicity compared with the unguided baseline, demonstrating controllable regulation of stability- and safety-related properties. 
Meanwhile, Table~\ref{tab:ablation_permeability} shows that the membrane-permeable variant substantially increases the permeability proxy relative to the baseline, confirming the effectiveness of guidance in steering permeability-related behavior. 
Overall, these results highlight that explicit guidance is essential for directing cyclic peptide generation toward desired physicochemical objectives.

\begin{table}[h]
\centering
\caption{Ablation on membrane permeability proxy. Higher values indicate better permeability.}
\label{tab:ablation_permeability}
\begin{tabular}{lc}
\toprule
\textbf{Method} & \textbf{Permeability proxy} $\uparrow$ \\
\midrule
Membrane-permeable (guided) & \textbf{0.107} \\
Multi-property (guided)           & -0.051 \\
Base (no guidance)           & 0.042 \\
\bottomrule
\end{tabular}
\end{table}

\section{Limitations and Ethical Considerations}

\textbf{Data and supervision constraints}
APCyc is trained on CPCore \citep{yang2025cpsea}, a curated CPSea subset derived from large-scale structure prediction resources. Although filtering and clustering improve data quality, the complexes may still inherit prediction noise, imperfect pocket geometries, and target-dependent biases. In addition, CPCore mainly covers classical cyclization chemistries, including backbone amide, disulfide, and isopeptide linkages. Generalization to alternative crosslinkers, bicyclic topologies, or non-canonical stapling chemistries may therefore require additional structural data and broader chemical annotations.

\noindent \textbf{Intended use, dual-use, and uncertainty}
APCyc is intended to support benign therapeutic discovery by proposing cyclic peptide candidates for expert review and experimental validation, rather than serving as a source of clinical recommendations. Because generative models for bioactive molecule design may be misused to search for harmful activity \citep{urbina2022dual}, practical deployment should remain within controlled research settings, follow biomedical governance and biosafety procedures, and include risk assessment, domain-expert oversight, and safety screening before synthesis \citep{topol2019high,schneider2020rethinking,pannu2024prioritizing}. In addition, safety-related objectives such as immunogenicity rely on computational predictors and may be uncertain or miscalibrated \citep{begoli2019need}. These predictions should therefore be treated as preliminary signals and complemented with additional models, expert assessment, experimental validation, and uncertainty reporting where possible \citep{amrhein2019scientists}.

\section{Conclusions}
We present APCyc, a target-aware latent diffusion framework that enables automated cyclization and property-informed design for cyclic peptides. APCyc dynamically infers pocket-adaptive linkage types and sites conditioned on the receptor context, and simultaneously optimizes multiple drug-like properties using the proposed Bayesian Posterior Guidance. As a result, APCyc addresses the critical challenge of controllable cyclic peptide generation under complex property constraints, providing a robust foundation and new insights for the design of next-generation therapeutic peptides.

\begin{acks}
This work was supported by the Guangdong Basic and Applied Basic Research Foundation
(2026A1515011793), and the Youth S\&T Talent Support Programme of Guangdong Provincial Association for Science and Technology (SKXRC2025467).
\end{acks}

\section*{GenAI Disclosure}

Parts of the manuscript preparation, including language refinement and editing, were assisted by generative AI tools. All scientific content, experimental design, implementation, and conclusions were developed and verified solely by the authors.

\bibliographystyle{ACM-Reference-Format}
\bibliography{reference}

\appendix

\section{Proof of \hyperref[prop:equivariance]{Proposition~\ref*{prop:equivariance}}}
\label{app:equivariance_proof}

\begin{proof}
Let $g \in \mathrm{SE}(3)$ denote a rigid transformation acting on equivariant
coordinates as $g(\mathbf{x}) = \mathbf{R}\mathbf{x} + \mathbf{t}$ with
$\mathbf{R} \in \mathrm{SO}(3)$. Invariant latent features remain unchanged.

\paragraph{Invariant topology prediction}
The pair representation in~\eqref{eq:pair_rep} depends only on invariant node
latents and Euclidean distances, which are preserved under rigid motions.
Therefore the predicted linkage matrix $\hat{\mathbf S}_t$ obtained via the
masked Softmax in~\eqref{eq:topo_softmax} is invariant:
\[
\hat{\mathbf S}_t(g \cdot \mathcal{Z}_t)
= \hat{\mathbf S}_t(\mathcal{Z}_t).
\]

\paragraph{Invariant Topology Injection}
The injected edge features, bias, and gating terms defined in
Eqs.~\eqref{eq:edge_inj}--\eqref{eq:gate_inj} depend only on invariant distances,
$\hat{\mathbf S}_t$, and scalar schedules. Hence all modulation coefficients
remain invariant under SE(3) transformations.

\paragraph{Equivariance of coordinate updates}
The AM-EGNN coordinate update takes the standard form
\[
\tilde z'_{i,t}
=
\tilde z_{i,t}
+
\sum_{j} \alpha_{ij,t}\,
(\tilde z_{i,t} - \tilde z_{j,t}),
\]
where the scalar weights $\alpha_{ij,t}$ depend only on invariant quantities.
Applying $g$ to both sides yields
\[
\tilde z'_{i,t}(g \cdot \mathcal{Z}_t)
=
g \cdot \tilde z'_{i,t}(\mathcal{Z}_t),
\]
which proves SE(3)-equivariance.

\paragraph{Conclusion}
Since topology prediction is invariant and coordinate updates are equivariant,
the overall topology injection mechanism preserves SE(3)-equivariance.
\end{proof}

\section{Variational Autoencoder Training Loss}
\label{vae}
Following \citet{kong2024full}, we train a joint sequence--structure variational autoencoder that maps a peptide
$\mathcal{P}=\{(s_i,\mathbf{X}_i)\}_{i=1}^{L}$ to residue-level latents
$\mathcal{Z}=\{(z_i,\tilde z_i)\}_{i=1}^{L}$,
where $z_i$ is $\mathrm{E}(3)$-invariant and $\tilde z_i$ is $\mathrm{E}(3)$-equivariant.
The decoder reconstructs $\hat s_i$ and $\hat{\mathbf{X}}_i$.

\paragraph{Overall objective.}
Following Sec.~\ref{sec:training}, the autoencoder is optimized by
\begin{equation}
\mathcal{L}_{\text{AE}}
=\frac{1}{L}\sum_{i=1}^{L}\big(
\mathcal{L}_{\text{rec}}(i)+\mathcal{L}_{\text{KL}}(i)
\big).
\end{equation}

\paragraph{Reconstruction loss.}
We use cross-entropy for residue tokens and MSE for full-atom coordinates,
together with an auxiliary physical regularization:
\begin{equation}
\mathcal{L}_{\text{rec}}(i)
=
H(s_i,\hat s_i)
+
\|\mathbf{X}_i-\hat{\mathbf{X}}_i\|_2^2
+
\gamma\,\mathcal{L}_{\text{phys}}(i).
\end{equation}

\paragraph{Auxiliary physical loss.}
To better preserve realistic geometry, $\mathcal{L}_{\text{phys}}$ supervises
$\mathrm{C}_\alpha$ coordinates, bond lengths, and side-chain dihedral angles:
\begin{equation}
\mathcal{L}_{\text{phys}}(i)
=
\lambda_{\mathrm{CA}}\mathcal{L}_{\mathrm{CA}}(i)
+
\lambda_{\mathrm{bond}}\mathcal{L}_{\mathrm{bond}}(i)
+
\lambda_{\mathrm{angle}}\mathcal{L}_{\mathrm{angle}}(i).
\end{equation}

\paragraph{Latent regularization.}
The invariant latent $z_i$ is regularized toward a standard Gaussian,
while the equivariant latent $\tilde z_i$ is anchored to the reference
$\mathrm{C}_\alpha$ coordinate $\mathbf{r}_i$:
\begin{equation}
\mathcal{L}_{\text{KL}}(i)
=
D_{\mathrm{KL}}\!\left(q_\phi(z_i)\,\|\,\mathcal{N}(\mathbf{0},\mathbf{I})\right)
+
\eta\,\|\tilde z_i-\mathbf{r}_i\|_2^2.
\end{equation}

\paragraph{Default weights.}
We set $\lambda_{\mathrm{CA}}=1.0$, $\lambda_{\mathrm{bond}}=1.0$, and
$\lambda_{\mathrm{angle}}=0.5$ in experiments, and use a small $\eta$
to stabilize equivariant latent learning.

\section{Metrics Implementation}
\label{metrics}
\subsection{Inputs and identifiers}\label{app:prop:input}
Each candidate is provided as a complex PDB (optionally gzipped). We assume a two-chain convention in the complex:
the peptide is chain \texttt{L} and the receptor/target protein is chain \texttt{R}.
We use the PDB filename stem as the peptide identifier \texttt{id}.
All sequence-based predictors operate on the peptide-chain sequence extracted from chain \texttt{L};
structure-based descriptors operate on the chain-\texttt{L} structure; affinity evaluates the \texttt{L}--\texttt{R} interface.

\subsection{Solubility (CamSol intrinsic)}\label{app:prop:sol}
To assess solubility, we employ the CamSol-intrinsic predictor on the peptide sequence and record the returned scalar
\texttt{camsol\_score} (higher indicates better solubility):
\begin{equation}
S_{\mathrm{sol}} := \texttt{camsol\_score}.
\end{equation}

\subsection{Enzymatic stability / proteolysis susceptibility (ProsperousPlus)}\label{app:prop:proteolysis}
To quantify proteolysis susceptibility, we employ ProsperousPlus in prediction mode to estimate cleavage propensity
under a fixed panel of proteases
\begin{equation}
\mathcal{P}=\{\texttt{S01.001},\texttt{A01.001},\texttt{C01.060},\texttt{M10.003},\texttt{M10.008}\}.
\end{equation}
Because ProsperousPlus expects a fixed-length context, we convert each peptide sequence into a set of 8-mer windows.
Specifically, after uppercasing and replacing non-canonical characters with \texttt{-}, we generate length-8 windows with stride 2,
ensuring the C-terminal window is included; if the number of windows exceeds $30$, we subsample evenly to at most $30$ windows.
For each window set, ProsperousPlus outputs per-position cleavage probabilities (column \texttt{pro}) for each protease.
We aggregate these raw probabilities by averaging within each protease and then summing across proteases:
\begin{equation}
\texttt{proteolysis\_sumscore}
~:=~
\sum_{p\in\mathcal{P}}
\left(
\frac{1}{N_p}\sum_{k=1}^{N_p} \texttt{pro}_{p,k}
\right),
\end{equation}
where $N_p$ is the number of probability entries produced for protease $p$ across all windows/positions.
Smaller values imply lower cleavage propensity (i.e., higher stability), and we therefore define a higher-is-better stability score as
\begin{equation}
S_{\mathrm{prot}} := -\,\texttt{proteolysis\_sumscore}.
\end{equation}

\subsection{Immunogenicity proxy (BigMHC EL/IM)}\label{app:prop:immuno}
To assess immunogenicity, we run BigMHC on a specified MHC-I allele (default: \texttt{HLA-A*02:01})
and obtain two outputs per peptide: \texttt{el} (presentation likelihood) and \texttt{im} (immunogenicity likelihood).
Peptides of length 8--11 are scored directly. For peptides longer than 11 residues, we enumerate sliding windows of lengths
$8,9,10,11$, score each window, and average the top-$K$ window scores (we use $K{=}3$) for EL and IM separately.
We report the raw immunogenicity log-score, where lower values indicate lower predicted immunogenicity, and use its sign-flipped form for internal surrogate training and guidance:
\begin{align}
I_{\mathrm{imm}} &:=
\log(\max(\texttt{el},\epsilon)) + \log(\max(\texttt{im},\epsilon)),\\
S_{\mathrm{imm}} &:= - I_{\mathrm{imm}}.
\end{align}
Thus all guidance objectives follow a higher-is-better convention.

\subsection{Physicochemical descriptors from peptide-chain structure}\label{app:prop:physchem}
To obtain structure-derived physicochemical descriptors, we analyze the peptide chain \texttt{L} extracted from the complex PDB.
We compute:
(i) \textbf{GRAVY} as the average Kyte--Doolittle hydropathy over the extracted peptide sequence;
(ii) \textbf{TPSA} and \textbf{logP} using RDKit descriptors computed from the peptide-chain PDB block;
(iii) the heavy-atom-normalized polar surface area
\begin{equation}
\texttt{rTPSA} := \frac{\texttt{TPSA}}{\#\text{heavy atoms}}.
\end{equation}

\subsection{Permeability proxy score}\label{app:prop:pproxy}
We consolidate the above descriptors into a single permeability proxy. Since the raw descriptors have different scales,
we first normalize each descriptor by dataset-level $z$-scoring:
\begin{equation}
z(x)=\frac{x-\mu(x)}{\sigma(x)},
\end{equation}
where $\mu(\cdot)$ and $\sigma(\cdot)$ are computed over the dataset (ignoring missing values; using population standard deviation).
We then define
\begin{equation}
P_{\mathrm{proxy}} :=
z(\texttt{logP}) + z(\texttt{GRAVY}) - z(\texttt{TPSA}) - z(\texttt{rTPSA}),
\end{equation}
which corresponds to higher hydrophobicity and lower polar surface area being more favorable.
\subsection{Binding affinity: Vina score-only evaluation}\label{app:prop:vina}

Binding affinity is estimated by evaluating the \emph{given} complex pose without redocking.
Chains \texttt{L} and \texttt{R} are extracted, converted to PDBQT via MGLTools.
AutoDock Vina is executed in \texttt{--score\_only} mode:
\begin{center}
\texttt{vina --receptor R.pdbqt --ligand L.pdbqt --center\_* c\_* --size\_* s\_* --score\_only}
\end{center}

The scoring box is defined from ligand bounds with padding $p$ (default $10$~\AA):
\begin{equation}
c_x=\tfrac{x_{\min}+x_{\max}}{2}, \quad
s_x=\max\!\big((x_{\max}-x_{\min})+p,\,1.0\big),
\end{equation}
and analogously for $y,z$.
The reported affinity (kcal/mol) is recorded as \texttt{vina\_score}.

\subsection{Binding affinity: RosettaScripts interface}\label{app:prop:rosetta}
In addition to Vina, we compute a Rosetta interface energy proxy with RosettaScripts.
Unlike the example FastRelax workflow, we do not perform additional relaxation in this pipeline; instead,
we directly evaluate the input complex using the \texttt{ref2015} scorefunction and a minimal protocol
(\texttt{interface\_analyze.xml}) that applies a single \texttt{Ddg} filter with \texttt{jump=1}, \texttt{repeats=1}, \texttt{repack=0} (no repacking/minimization), and scorefunction \texttt{ref2015}.
We run rosetta\_scripts.* in score-only mode and extract the scalar \texttt{dg} reported by the filter as \texttt{rosetta\_dG}.

\subsection{Affinity normalization and combined score}\label{app:prop:affinity}
Because Vina and Rosetta energies have different scales, we convert each to a higher-is-better normalized term via $z$-scoring and sign flip. We then define
\begin{align}
A_{\mathrm{ros}}&:=-z(\texttt{rosetta\_dG}),\qquad
A_{\mathrm{vina}}:=-z(\texttt{vina\_score}),\\
A_{\mathrm{aff}}&:=\tfrac{1}{2}A_{\mathrm{ros}}+\tfrac{1}{2}A_{\mathrm{vina}},
\end{align}
and use $A_{\mathrm{aff}}$ as \texttt{affinity} property to train the surrogate.

\subsection{Unified optimization direction}
For surrogate training and gradient-based guidance, all property scores are
converted to a \emph{higher-is-better} convention to ensure a consistent
optimization direction across objectives.
This transformation is used only for internal optimization.
Reported results follow the standard biochemical interpretation of each metric.

\section{Guidance Weights for Therapeutic Objectives}
\label{detail}

APCyc uses five Bayesian posterior-guided sampling configurations, \textbf{Base}, \textbf{Aff.}, \textbf{Prot.}, \textbf{Perm.}, and \textbf{Multi}, by assigning property-specific guidance weights. Table~\ref{tab:guidance_weights} lists the coefficients applied to the joint energy function in Sec.~\ref{sec:guidance}.

\begin{center}
\captionof{table}{Guidance weights for different therapeutic design objectives.}
\label{tab:guidance_weights}
\begingroup
\footnotesize
\renewcommand{\arraystretch}{1.08}
\setlength{\tabcolsep}{4.5pt}
\begin{tabular}{lccccc}
\toprule
\textbf{Obj.} & \textbf{Aff.} & \textbf{Prot.} & \textbf{Perm.} & \textbf{Sol.} & \textbf{Imm.} \\
\midrule
Base  & 0  & 0 & 0  & 0 & 0 \\
Aff.  & 10 & 0 & 0  & 0 & 0 \\
Prot. & 8  & 8 & 0  & 2 & 0 \\
Perm. & 8  & 2 & 10 & 2 & 0 \\
Multi & 6  & 4 & 4  & 3 & 5 \\
\bottomrule
\end{tabular}
\endgroup
\end{center}

\end{document}